\documentclass[letterpaper, 10 pt, conference]{ieeeconf}

\IEEEoverridecommandlockouts

\usepackage{graphicx}
\usepackage{amsmath}
\usepackage{amssymb}
\usepackage{algorithm}
\usepackage[noend]{algpseudocode}
\usepackage{tikz}
\usetikzlibrary{shapes.geometric, arrows}
\usepackage{xcolor}
\usepackage{url}
\usepackage{subfig}
\usepackage{cite}  
\usepackage{caption}

\DeclareCaptionFont{eightpt}{\fontsize{9pt}{9pt}\selectfont}

\newcommand{\Cpp}{C\raise.08ex\hbox{\tt ++}\xspace}

\newtheorem{lemma}{Lemma}
\newtheorem{theorem}{Theorem}

\newtheorem{defin}{Definition}

\newcommand{\ignore}[1]{}

\newcommand{\addMP}[1]{}
\newcommand{\addExample}[1]{}

\newboolean{HIDENOTES}
\setboolean{HIDENOTES}{false}

\ifthenelse{\boolean{HIDENOTES}}{
    \newcommand{\OS}[1]{{}}
    \newcommand{\ST}[1]{{}}
}{
    \newcommand{\OS}[1]{\textcolor{orange}{#1}}
    \newcommand{\ST}[1]{{\textcolor{blue}{#1}}}    
}

\newcommand{\Xdelta}{\mathcal{X}^{\delta}}
\newcommand{\Xnotdelta}{\mathcal{X}^{\overline{\delta}}}
\newcommand{\Xdeltafree}{\mathcal{X}^{\delta}_{\text{free}}}
\newcommand{\Xdeltaforb}{\mathcal{X}^{\delta}_{\text{forb}}}
\newcommand{\Xfree}{\mathcal{X}_{\text{free}}}
\newcommand{\Xforb}{\mathcal{X}_{\text{forb}}}

\newboolean{CONFERENCE}
\setboolean{CONFERENCE}{false}

\ifthenelse{\boolean{CONFERENCE}}{
    \newcommand{\conference}[2]{#1}
}{
    \newcommand{\conference}[2]{{#2}}
}

\definecolor{my_cyan}{HTML}{D7F3FF}
\definecolor{my_gray}{HTML}{D9D9D9}
\definecolor{my_red}{HTML}{FFD7D7}
\definecolor{my_dark_gray}{HTML}{AAAAAA}

\tikzstyle{startstop} = [rectangle, rounded corners, minimum width=3cm, minimum height=1cm, text centered, draw=black, fill=red!30]
\tikzstyle{process} = [rectangle, minimum width=3cm, minimum height=1cm, text centered, draw=black, fill=orange!30]
\tikzstyle{arrow} = [thick,->,>=stealth]

\title{\LARGE \bf
{From Configuration-Space Clearance to Feature-Space Margin:}\\
{Sample Complexity in Learning-Based Collision Detection}
}

\author{Sapir Tubul$^{1}$, Aviv Tamar$^{2}$, {Kiril} Solovey$^{2}$, and Oren Salzman$^{1}$%
\thanks{$^{1}$S. Tubul and O. Salzman are with the Faculty of Computer Science, Technion - Israel Institute of Technology, Israel.}%
\thanks{$^{2}$A. Tamar and K. Solovey are with the Faculty of Electrical and Computer Engineering, Technion - Israel Institute of Technology, Israel.}%
\thanks{Email: \{tubulsapir, osalzman\}@cs.technion.ac.il, \{avivt, {kirilsol}\}@ee.technion.ac.il}%
}

\begin{document}

\maketitle
\thispagestyle{empty}
\pagestyle{empty}

\begin{abstract}

Motion planning is a central challenge in robotics, with learning-based approaches gaining significant attention in recent years. Our work focuses on a specific aspect of these approaches: using machine-learning techniques, particularly Support Vector Machines (SVM), to evaluate whether robot configurations are collision free, an operation termed ``collision detection''. Despite the growing popularity of these methods, there is a lack of theory supporting their efficiency and prediction accuracy. 
This is in stark contrast to the rich {theoretical} results of machine-learning methods in general and of {SVMs} in particular.
Our work bridges this gap by analyzing the sample complexity of an SVM classifier for learning-based collision detection in motion planning. We bound the number of samples needed to achieve a specified accuracy at a given confidence level. This result is stated in terms relevant to robot motion-planning such as the system's clearance. 
Building on these theoretical results, we propose a collision-detection algorithm that can also provide statistical guarantees on the algorithm's error in classifying robot configurations as collision-free or not. 
\end{abstract}

\section{Introduction}

Robot motion-planning (MP) is a fundamental problem in robotics where, in its simplest form, we are tasked with computing a collision-free path for a robot system navigating among a set of known obstacles~\cite{lavalle2006planning,HSS17,latombe2012robot}.
While a multitude of approaches and algorithms exist for MP (see~\cite{karaman2011sampling,elbanhawi2014sampling,ZuckerRDPKDBS13,CohenCL14} for a very partial list), they typically share a common algorithmic building block called \emph{collision detection}~(CD). CD tests whether the robot collides with an obstacle when placed at a given configuration (a set of parameters that determine the robot's position and orientation).
%

CD is often considered the computational bottleneck in MP algorithms~\cite{lavalle2006planning,Salzman19,Hauser15}. 
Roughly speaking, this is because 
(i)~each CD operation is computationally expensive as it corresponds to a set of geometric tests between the robot's and obstacles' geometric representations~\cite{latombe2012robot}
and
(ii)~it is used numerous times within MP algorithms to evaluate whether configurations and local motions are collision-free or not~\cite{kavraki1996probabilistic,lavalle2001randomized}.

Some MP algorithms attempt to reduce the number of CD calls using different algorithmic techniques such as lazy evaluation~\cite{Hauser15,DellinS16,MandalikaCSS19}. 
An alternative approach, which is the focus of this work, is speeding up CD via machine-learning (ML) techniques in which  exact CD algorithms are replaced with faster statistical methods which we refer to as \emph{Learned CD} or LCD~\cite{das2020learning,ZhiDY22,munoz2023collisiongp,DasY20,chase2020neural}. 

Such methods were successfully integrated into MP algorithms~\cite{wang2021survey,mcmahon2022survey}. However, to the {best} of our knowledge, {LCDs} are used without any guarantees on the result obtained: LCD can produce both false positive and false negative results. I.e., they may state that a configuration is in collision when it is not and vice versa.
%

The ML community has a rich set of tools and analysis providing statistical guarantees on the quality of different algorithms~\cite{shalev2014understanding}. A prominent example is the notion of \emph{sample complexity} which refers to the number of training examples or data points required by an ML algorithm to learn a task with a desired level of accuracy and confidence. 
Understanding sample complexity is crucial for ML practitioners to design efficient algorithms and ensure they have sufficient data for their models to learn effectively.
However, it is not clear a priori how to express such results for robot MP problems. Roughly speaking, this is because both (i)~the parameters which are used to express ML and MP problems  differ and
(ii)~the type of guarantees typically offered by the two families of algorithms differ.
%

In this work we bridge this gap and provide statistical guarantees on the sample complexity of a commonly used LCD for MP which is based on the well-known SVM algorithm~\cite{shalev2014understanding,cristianini2000introduction}. 
After discussing related work (Sec.~\ref{sec:related-work}), we provide the necessary background to understand our results (Sec.~\ref{sec:background}). 
We then present our main contribution which allows us to leverage existing results regarding sample complexity of SVM for LCDs (Sec.~\ref{sec:theory}).
These results cannot be used as-is since they require unrealistic assumptions for the analysis (specifically, that we only call the LCD for samples that are ``not too close'' to obstacles).
Thus, in Sec.~\ref{sec:applications}, we present an algorithm which, given an error $\varepsilon$ and confidence level $1-\xi$, returns an LCD that guarantees, with
confidence at least $1-\xi$, that the error of the CD on \emph{any} random configuration does not exceed $\varepsilon$.
After some experimental results (Sec.~\ref{sec:eval}) highlighting properties key to our theory, we conclude (Sec.~\ref{sec:future}) with a description of limitations of our work and potential future {directions.}

\section{Related Work}
\label{sec:related-work}

\subsection{Learning-Based Collision Detection in Motion Planning}
Learning-based approaches for CD have gained significant attention for improving the efficiency and adaptability of MP algorithms. These methods aim to reduce the computational burden of traditional CD by leveraging ML techniques to approximate the geometry of the configuration space.

A prominent example is using Support Vector Machines (SVMs). 
Das and Yip~\cite{das2020learning} introduced the Fastron algorithm, which employs SVM and active learning~\cite{settles2009active} to quickly identify forbidden configurations even when the obstacles move in the environment. 
This was later enhanced by incorporating forward kinematics into the learning process~\cite{das2020forward}, improving both accuracy and efficiency of LCD.

Alternative ML techniques beyond SVMs have also been introduced. 
For example, Yu and Gao~\cite{yu2021reducing} proposed using Graph Neural Networks (GNNs)~\cite{zhou2020graph} to reduce the computational cost of CD in MP.
This was later extended to dynamic environments by incorporating temporal encoding~\cite{zhang2022learning}.

Unfortunately, none of the aforementioned LCDs provide formal guarantees on the classification accuracy of new configurations.

{Looking beyond CD, it is worth noting that recent works used ML for configuration-space representations: 
Li et al.~\cite{Li2024} introduced a method for configuration space distance fields, 
while Koptev et al.~\cite{Koptev2022} developed a neural implicit function for reactive manipulator control.
These approaches focus on approximating configuration space distances or using neural models for implicit 
collision representations.}

\subsection{Sample Complexity}
Sample complexity refers to the minimum number of training samples that an ML algorithm needs in order to successfully learn a target function or achieve a desired level of performance. It is typically expressed as a function of the desired accuracy ($\varepsilon$) and confidence ($\xi$).
Sample complexity is closely related to the notion of VC dimension~\cite{vapnik2015uniform} which, roughly speaking, measures the complexity of a hypothesis set or classification model.

For binary classification problems, the fundamental theorem of statistical learning states that the sample complexity is linearly related to the VC dimension of the hypothesis class~\cite{shalev2014understanding}. 
Specifically, for a hypothesis class $H$ with VC dimension $d$, the sample complexity $m(\varepsilon, \xi)$ to achieve an error of at most $\varepsilon$ with probability at least $1-\xi$ is bounded by 
$
m(\varepsilon, \xi) = O\left({(d + \log(1/\xi))} / {\varepsilon^2}\right)
$.

The VC dimension of linear SVMs in an $n$-dimensional space is $n+1$, immediately leading to sample complexity bounds for these classifiers~\cite{shalev2014understanding}.
Importantly, for certain  settings which contain some \emph{margin} between different classes (see Sec.~\ref{sec:theory} for a precise definition), the sample complexity of SVMs can be independent of the input space's dimension.

Sample complexity was recently used for robotics-related problems.
For instance, recent work~\cite{TsaoSP20, DayanSPH23}  used geometric techniques to such  bounds for PRM~\cite{kavraki1996probabilistic} and related methods to achieve a desired solution quality using deterministic sampling. Another work~\cite{shaw2024practicalfinitesamplebounds} considers probabilistic sample complexity bounds for randomized sampling in the context of task and motion planning~\cite{garrett2021integrated}.

\section{Algorithmic Background}
\label{sec:background}

\subsection{Classification}
Classification is a supervised-learning technique where an ML algorithm learns to categorize input data into predefined classes or categories based on labeled training examples. We provide a high-level description of the problem and refer the reader to standard textbooks (e.g.,~\cite{shalev2014understanding}) for additional details.

The input to the problem is a training set~
$S =~\{(x_1, y_1), \ldots, (x_m, y_m)\}$ of $m$ pairs,
where~$x_i \in \mathcal{X}$ represents an input from a domain $\mathcal{X}$
and $y_i \in \mathcal{Y}$ is its corresponding label from a label space $\mathcal{Y}$. 
Roughly speaking, the objective of the learning process is to find a function~$h: \mathcal{X} \rightarrow \mathcal{Y}$, called a \emph{hypothesis}, that approximates the true underlying function $f: \mathcal{X} \rightarrow \mathcal{Y}$ which generated the labels.

To evaluate the performance of a learned model, a key concept is 0-1 loss: Given a distribution $\mathcal{D}$ over~$\mathcal{X}$, a target function $f$, and a hypothesis $h$, the 0-1~loss~$\mathcal{L}_{\mathcal{D},f}(h)$ is defined as:~$
    \mathcal{L}_{\mathcal{D},f}(h) :=~\mathbb{P}_{x\sim \mathcal{D}}\left[h(x) \neq f(x)\right]
$.

This 0-1 loss represents the probability of misclassification when the prediction rule $h$ is applied to an input $x$ drawn randomly according to the distribution $\mathcal{D}$. In other words, it quantifies the likelihood that our learned model will make an incorrect prediction on new, unseen data from the same distribution as our training set. 
In the rest of the paper, when we use the term ``error'', we refer to the 0-1 loss.

\subsection{Support Vector Machines}
Support Vector Machines (SVMs) are a class of classification algorithms. In the context of binary classification, which is relevant to our setting of LCD, SVMs work by finding the hyperplane that best separates two classes of data points in a high-dimensional space. The ``support vectors'' are the data points that lie closest to this decision boundary. SVMs can efficiently handle non-linear classification using \emph{kernels}, which implicitly map inputs into high-dimensional feature spaces. This allows SVMs to find complex decision boundaries while maintaining computational efficiency.

In our work, we specifically employ Hard-SVM, a variant of SVM that assumes the data is linearly separable in the feature space. For a training set $\{(x_i, y_i)\}_{i=1}^m$ where $x_i \in \mathbb{R}^d$ and $y_i \in \{-1,+1\}$, the Hard-SVM optimization problem is formulated as
$\min_{w, b}~\frac{1}{2}\|w\|^2~\text{s.t.}~y_i(w \cdot x_i + b) \geq 1$.
Here,~$w$ is the normal vector to the hyperplane and $b$ is the bias term. It can be shown that Hard-SVM finds the hyperplane with the largest margin between the two classes~\cite{shalev2014understanding}.

\section{Theoretical Framework}
\label{sec:theory}

In this section, we present our theoretical framework for analyzing sample complexity in ML-based CD. 
We begin by introducing key definitions and then present two important theorems that form the foundation of our approach. 

\subsection{Definitions}
Let $\mathcal{X}$ denote the C-space of a robot, partitioned into the free space, $\Xfree$, and the forbidden space, $\Xforb$.
For simplicity, we assume that $\mathcal{X}$ is an axis-aligned $d$-dimensional unit hypercube. Namely, $\mathcal{X} := [0, 1]^d$.

\begin{defin}[Clearance]
\label{def:clearance}
For a configuration $\mathbf{x} \in \mathcal{X}$, we define its clearance $\text{cl}(\mathbf{x})$ as:
\[
\text{cl}(\mathbf{x}) := \begin{cases}
\min_{\mathbf{x'} \in \Xforb} \| \mathbf{x}-\mathbf{x'} \|_2, & \text{if } \mathbf{x} \in \Xfree, \\
\min_{\mathbf{x'} \in \Xfree} \| \mathbf{x}-\mathbf{x'} \|_2, & \text{if } \mathbf{x} \in \Xforb.
\end{cases}
\]
where $\| \cdot \|_2$ denotes the Euclidean norm in the C-space.
\end{defin}

\begin{defin}[$\delta$-interior $\&$ $\delta$-boundary]
\label{def:interior}
We define the $\delta\text{-interior}$ of a C-space $\mathcal{X}$  as:
$
\Xdelta := 
    \{\mathbf{x} \in \mathcal{X}~\vert~\text{cl}(\mathbf{x}) \geq~\delta \}
$.
Similarly, the $\delta$-boundary of a C-space $\mathcal{X}$ is defined  as~$\Xnotdelta :=~\mathcal{X} \setminus \Xdelta$.
Finally, for convenience we set~$\Xdeltafree :=~\Xdelta \cap \Xfree$
and
$\Xdeltaforb := \Xdelta \cap \Xforb$.

\end{defin}

See {Fig.~\ref{fig:workspace_delta}} for a visualization of the notions of clearance, $\delta$-interior and $\delta$-boundary.

\begin{figure}[t]
    \centering
    \subfloat[]{%
        \includegraphics[width=0.42\linewidth]{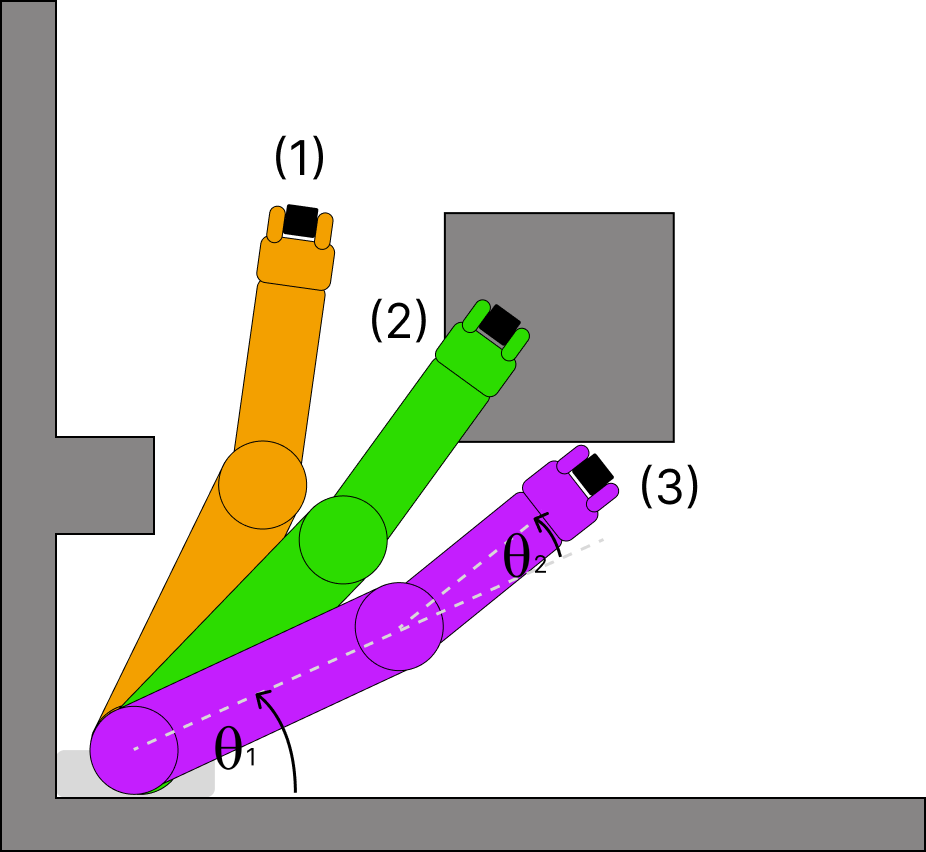}%
        \label{fig:workspace}%
    }%
    \hfill
    \subfloat[]{%
        \includegraphics[width=0.27\linewidth]{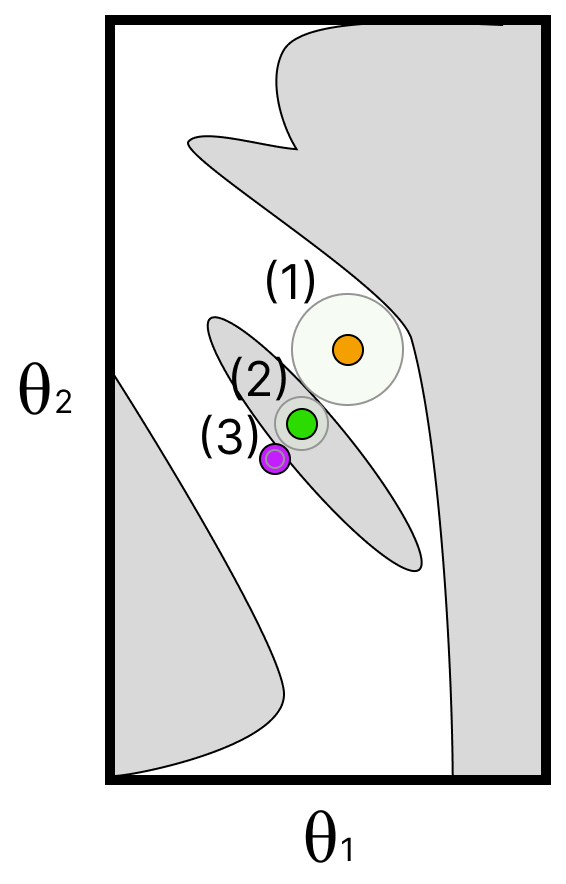}%
        \label{fig:cspace}%
    }%
    \hfill
    \subfloat[]{%
        \includegraphics[width=0.27\linewidth]{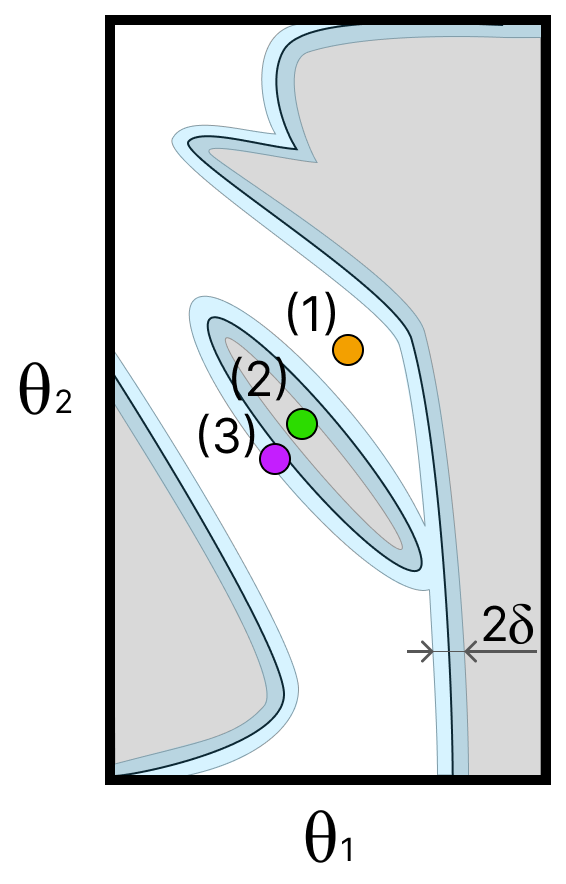}%
        \label{fig:delta_boundary}%
    }%
    \vspace{- 1mm}
    \caption{
    (a) Workspace with a two-link robot in three different configurations, each with its own color and label: orange~(1), green~(2) and purple (3). 
    (b) C-space representation depicting~$\Xfree$ 
    (\fcolorbox{black}{white}{\rule{0pt}{3pt}\rule{3pt}{0pt}})
    and~$\Xforb$
    (\fcolorbox{black}{my_gray}{\rule{0pt}{3pt}\rule{3pt}{0pt}}),
    as well as configurations whose color corresponds to the robot's color in that configuration. Circles around each configuration correspond to its clearances. 
    (c)~$\delta\text{-interior}$~$\Xdelta$ 
    (\fcolorbox{black}{white}{\rule{0pt}{3pt}\rule{3pt}{0pt}}~and~\fcolorbox{black}{my_gray}{\rule{0pt}{3pt}\rule{3pt}{0pt}})
    and $\delta$-boundary~$\Xnotdelta$~(\fcolorbox{black}{my_cyan}{\rule{0pt}{3pt}\rule{3pt}{0pt}}).
    Configurations (1) and (2) lie in~$\Xdelta$, while~(3) is in~$\Xnotdelta$.
    }
    \label{fig:workspace_delta}
    \vspace{-6mm}
\end{figure}

As we will see shortly, it will be convenient to discretize the continuous C-space. Thus, we overlay a $d$-dimensional axis-aligned grid over~$\mathcal{X}$ with cell size $1/n$ and call $n$ the \emph{resolution}. 

\begin{defin}
\label{cell}
Let $\mathbf{i} =\langle i_1,i_2,\ldots,i_d \rangle \in [n]^d$ be a \(d\)-dimensional index.
We define \emph{\(\textsc{Cell}_{\mathbf{i}}\)} to be a grid cell s.t.
\[
    \textsc{Cell}_{\mathbf{i}} 
    := 
    \{(x_1, \ldots,x_d) \mid \forall j \in [d],~x_j 
        \in 
    \left[
        \frac{{i_j-1}}{n}, \frac{{i_j}}{n}
    \right)
    \}.
\]
Similarly, we define 
\emph{$\textsc{Center}_{\mathbf{i}}$}$\in \textsc{Cell}_{\mathbf{i}}$
to be the configuration lying at the center of \(\textsc{Cell}_{\mathbf{i}}\).
Namely,
\[
    \textsc{Center}_{\mathbf{i}} 
    := 
    \left(
        \frac{{i_1-0.5}}{n}, 
        \frac{{i_2-0.5}}{n}, 
        \ldots,
        \frac{{i_d-0.5}}{n}
    \right).
\]
\end{defin}

It will also be useful to define the following function: 
Given $\sigma \in \mathbb{R}^{>0}$,
we set \(h_{\mathbf{i}}^\sigma:\mathcal{X}\rightarrow [0,1]\) as follows:
\[
h_{\mathbf{i}}^\sigma(\mathbf{x}) := 
        \exp\left(-\frac{\|\textsc{Center}_{\mathbf{i}} - \mathbf{x}\|^2}{\sigma^2}\right).
\]
Now, we define the following mapping 
$\phi_\sigma : \mathcal{X} \rightarrow [0,1]^{n^d}$:
{
\begin{align}
    \phi_\sigma(\mathbf{x}) := 
        (&h_{1,1,\ldots,1}^\sigma(\mathbf{x}),  
        \ldots, 
        h_{n,1,\ldots,1}^\sigma(\mathbf{x}), \notag \\
        %
        %
        &\ldots \notag \\
        &h_{n,n,\ldots,1}^\sigma(\mathbf{x}),  
        \ldots, 
        h_{n,n,\ldots,n}^\sigma(\mathbf{x}))^T.
        \label{eq:feature_vec}
\end{align}
}
When clear from the context we will omit $\sigma$ from $\phi_\sigma$ and~$h^\sigma$ and simply write~$\phi, h$.

The mapping $\phi$ uses the $h$ functions to create a \emph{feature vector} for each configuration $x \in \mathcal{X}$. 
Conceptually, each element of this vector represents the ``influence" of a particular grid cell on $x$. This results in a rich, smooth representation of the C-space that captures both local and global structure, making it suitable for ML algorithms to work with.
Given such a mapping $\phi$, 
we define the \emph{feature space} of~$\mathcal{X}$ as~$\Phi:= \{ \phi(x) \mid x \in \mathcal{X} \} $.

Through the notion of clearance (Def.~\ref{def:clearance}), the $\delta$-interior~$\mathcal{X}^\delta$ (Def.~\ref{def:interior}) allows us to ensure a separation between points in~$\mathcal{X}^\delta \cap \ \Xfree$ 
and~$\mathcal{X}^\delta \cap \ \Xforb$.
Here, this separation is~${2}\delta$. 
The following definition establishes a similar concept of separation called \emph{margin} but in the feature space.

\begin{defin}[margin]
\label{def:margin}
Let $\phi$ be some mapping over~$\mathcal{X}$.
We say that there is a \emph{feature-space margin} $\gamma$ (for some~$\gamma >~0$) over the $\delta$-interior of $\mathcal{X}$
if there exists some vector $\mathbf{\alpha}$ with $|\mathbf{\alpha}| = 1$ s.t.
$\forall \mathbf{x} \in \mathcal{X}^\delta,~ f(\mathbf{x}) \cdot \mathbf{\alpha} \cdot \phi (\mathbf{x}) \geq \gamma$, 
where $f(\mathbf{x}) = +1$ if $\mathbf{x} \in \mathcal{X}_{\rm forb}$ and $f(\mathbf{x}) = -1$ if $\mathbf{x} \in \mathcal{X}_{\rm free}$.

Equivalently, we can express this as $\exists \mathbf{\alpha}$, such that,
(i)~if~$\mathbf{x} \in \mathcal{X}^\delta \cap \mathcal{X}_{\rm forb}$ then
$\mathbf{\alpha} \cdot \phi(\mathbf{x}) \geq \gamma$
and
(ii)~if~$\mathbf{x} \in~\mathcal{X}^\delta \cap~\mathcal{X}_{\rm free}$ then~$\mathbf{\alpha} \cdot \phi(\mathbf{x}) \leq -\gamma$.
\end{defin}

\subsection{Theorems}
Our first theorem establishes a connection between configurations in the $\delta$-interior of a C-space and the feature-space margin over this $\delta$-interior. 
This connection is key to our analysis as it allows us to use tools designed to analyze ML classification algorithms such as SVM by assuming that there exists some margin in~$\Phi$.

Hereafter, we  provide proof sketches for our theoretic results and refer the reader to
\conference{the extended version of our paper\footnote{\url{https://tinyurl.com/SampleComplexityInLearnedCD}.}}{the appendix} for the full proofs.

\begin{theorem}[Feature-Space Margin]
\label{thm:Feature-Space Margin}
For any $\delta {\in (0, \sqrt{d})}$, 
set~$n: = {{\sqrt{d}}/{\delta}}$ 
and~$\sigma^2 := {{2\delta^2}/{\ln(9n^d)}}$.
Given the feature mapping $\phi_\sigma$, 
then the set of feature-space points~$\{\phi_\sigma(x)~\vert~x\in \Xdeltafree \}$
and
$\{\phi_\sigma(x)~\vert~x\in \Xdeltaforb \}$
have a feature-space margin of at least $\gamma^* \geq 8 / \left( {9^{\frac{9}{8}} n^{\frac{5d}{8}}}\right)$.
\end{theorem}

\begin{proof}[sketch]    
We start by showing that our choice of~$n: = {{\sqrt{d}} / {\delta}}$ ensures that any cell that contains at least one configuration in the $\delta$-interior cannot contain configurations in both $\Xfree$ and $\Xforb$.
Using this property allows us to assign a label to each grid cell indicating whether it is considered in the $\delta$-boundary, 
in~$\Xdeltafree$ or 
in~$\Xdeltaforb$.
Moreover, using our assumption on the clearance $\delta$, we can show that cells labeled in~$\Xdeltafree$ are ``not too close'' to cells in~$\Xdeltaforb$.

Now, this cell labeling is used to define a vector $\alpha$ such that there is a feature-space margin of at least $\gamma^*$.
Conceptually, 
the crux of the proof lies in the fact that 
if~$x \in \Xdeltafree$ then the weight of the feature corresponding to the cell that contains $x$ outweighs the aggregate weight of all other cells because they are ``not too close''.
A similar argument holds when~$x \in \Xdeltaforb$.
\end{proof}

After establishing a connection between the clearance $\delta$ in the C-space~$\mathcal{X}$ and the feature-space margin over~$\mathcal{X}^\delta$, we use it to bound the sample complexity for learning a~CD over~$\mathcal{X}^\delta$. Importantly, we ignore how samples that lie outside~$\mathcal{X}^\delta$ will be labeled by this learned CD and will address this shortly.

\begin{theorem}[Sample Complexity for learning \( \Xdelta \)]
\label{thm:sample_complexity_delta}
The sample complexity of Hard-SVM, to ensure with at least~$1 - \xi$ confidence that the $0$-$1$ loss on samples drawn from~\( \Xdelta \) does not exceed \( \varepsilon \), is bounded  by:
\begin{equation}
\label{eq:sample_complexity}
m_{\Xdelta}(\varepsilon, \xi) \leq \frac{1}{\varepsilon^2} \left[
\frac{9^{\frac{9}{4}}}{4}\left({\frac{\sqrt{d}}{\delta}}\right)^{\frac{9d}{4}}+8\ln\left(\frac{2}{\xi}\right)
\right].
\end{equation}
\end{theorem}

Thm.~\ref{thm:sample_complexity_delta} can be interpreted as follows: 
given a clearance~$\delta \in~(0, \sqrt{d})$, 
error value $\varepsilon >0$
and confidence level~$1-\xi$,
if we sample $m_{\Xdelta}(\varepsilon, \xi)$ configurations (using any distribution~$\mathcal{D}$) in $\mathcal{X}^\delta$, accurately label each one if it is in $\mathcal{X}_{\text {free}}$ or in $\mathcal{X}_{\text {forb}}$ and run the Hard-SVM on this labelled data set, 
then the resultant classifier will guarantee with confidence level $1 - \xi$ that the classification error of any point sampled from $\mathcal{X}^\delta$ using  distribution~$\mathcal{D}$ is bounded by~$\varepsilon$.

\begin{proof}[sketch]
    The proof relies on Thm.~\ref{thm:Feature-Space Margin} and existing results (see, e.g.,~\cite[Thm.~15.4]{shalev2014understanding}) that bound the error of Hard-SVM for a feature space with a known margin.
    Together with careful algebraic manipulation and parameter choices, the proof is completed.
\end{proof}

\subsection{Discussion}
\label{subsec:discussion1}
A key caveat of Thm.~\ref{thm:sample_complexity_delta} is that it only bounds the classification  error over samples in the $\delta$-interior~$\Xdelta$ and provides no guarantees on how samples in the $\delta$-boundary~$\Xnotdelta$ are classified.
One approach to address this issue is to minimize~$\delta$ which will then maximize~$\Xdelta$ and allow us to accurately learn the structure of the C-space. Unfortunately, following Eq.~\ref{eq:sample_complexity}, this implies that the sample complexity dramatically grows and becomes, in the limit, unbounded. Formally, for any fixed $\varepsilon, \xi$ we have that
$\lim_{\delta \rightarrow 0} m_{\Xdelta}(\varepsilon, \xi) = \infty$.
On the other extreme, as~$\delta$ grows, the region for which we get no guarantees on the classification error of Hard-SVM grows and may exceed any error bound provided by a user.

In Sec.~\ref{sec:applications} we show how we can still use our results to suggest a learned CD that provides statistical guarantees on its classification error for any sample without the need to know C-space parameters such as clearance.

\section{LCD with statistical guarantees}
\label{sec:applications}

In this section we propose an algorithm that uses our theoretical results to produce a learned CD that provides statistical guarantees on its classification error.
We start (Sec.~\ref{subsec:LCD}) with an algorithm that receives as input some clearance $\delta$ and number of samples $m$ for the training set of the LCD. As we will see, some parameters impose constraints that are too strict, and the algorithm cannot output an LCD with corresponding statistical guarantees.
However, we show that there exist some parameters for which this algorithm succeeds. To this end, we show (Sec.~\ref{subsec:lcd2}) how the algorithm suggested in Sec.~\ref{subsec:LCD} can be used to obtain an LCD with the desired statistical guarantees.

\subsection{LCD with Guarantees---Foundation}
\label{subsec:LCD}
\begin{algorithm}[t]
\caption{$(\delta,m)$-Learning-based collision detection}
\label{algo:Learning Phase}
\begin{algorithmic}[1]
\Require 
    $\varepsilon$, 
    $\xi$, 
    $\delta$,
    $m$,
    $\mathcal{X} = [0,1]^d$,      
    \texttt{CD} 

\State $S \gets \textsc{Sample}(m, \mathcal{X}$) \label{line:sample}
\State $S_{\text{interior}} \gets  \{s \in S~\vert~\text{\texttt{CD}.cl}(s) > \delta \}$ \label{line:S interior}
\State $\varepsilon_{\Xdelta} \gets \textsc{GetInteriorError}(\varepsilon, \xi, |S_{\text{interior}}|, |S|)$ \label{line:epsilon interior}
\If{
    $\underbrace{\left(\varepsilon_{\Xdelta} \leq 0\right)}_{(C1)}$ 
    or 
    $\underbrace{\left(|S_{\text{interior}}| < m_{\Xdelta} (\varepsilon_{\Xdelta}, \xi)\right)}_{(C2)}$} \label{line:number of sample condition}
\State \Return \text{`Fail'} \label{line:fail}
\EndIf
\vspace{2mm}

\State \texttt{LCD} $\gets$ \textsc{HardSVM}$\left(\{(\phi(s), \texttt{CD}(s)) \mid s \in S_{\text{interior}}\}\right)$ \label{line:LCD}
\State \Return \texttt{LCD} \label{line:return LCD}
\end{algorithmic}
\end{algorithm}

The foundation of our LCD, outlined in Alg.~\ref{algo:Learning Phase}, receives as an input 
a desired error bound~$\varepsilon$ and confidence level~$1 - \xi$,
a clearance value~$\delta$, 
a number~$m$ of configurations to sample
and
an exact collision detector  \texttt{CD}.
The algorithm starts by sampling~$m$ configurations from the C-space~$\mathcal{X}$ (Line~\ref{line:sample})
and for each sample~$s$ computes its clearance $\text{cl}(s)$ using the exact collision detector.
The set of samples whose clearance is larger than~$\delta$ are stored in a set $S_{\text{interior}}$ (Line~\ref{line:S interior}).

$S_{\text{interior}}$ is used in the function \textsc{GetInteriorError} (Line~\ref{line:epsilon interior}) to estimate $|\Xdelta |$, the size of~$\Xdelta$. 
This is done by computing
$\hat{p}_{\Xdelta} := |S_{\text{interior}}| / |S|$, the empirical proportion of samples in the $\delta$-interior, 
and~$z_{\xi/2}$, the critical value of the standard normal distribution corresponding to the confidence level~$(1 - \xi)$~\cite{Wackerly2008}.
Using these values, \textsc{GetInteriorError} returns~$\varepsilon_{\Xdelta}$ which, as we will see, is the classification error on samples in the $\delta$-interior~$\Xdelta$ that is tolerable to obtain an overall  classification error of $\varepsilon$.
Specifically, it is computed~as
\begin{equation}
\label{eq:classification_error}
\varepsilon_{\Xdelta} = 
    \frac{
        \varepsilon - (1 - \hat{p}_{\Xdelta}) - z_{\xi/2} \cdot \sqrt{\frac{\hat{p}_{\Xdelta} (1 - \hat{p}_{\Xdelta})}{|S|}}
    }{
        \hat{p}_{\Xdelta} + z_{\xi/2} \cdot \sqrt{\frac{\hat{p}_{\Xdelta} (1 - \hat{p}_{\Xdelta})}{|S|}}
    }.
\end{equation}

As discussed in Sec.~\ref{subsec:discussion1} and we will formally show, a wrong choice of parameters $m$ and $\delta$ may cause the algorithm to fail to be able to compute a collision detector that satisfies the required bounds. 
This happens in two cases:
\begin{itemize}
    \item[(C1)] If $\delta$ is too big or if $m$ is too small, then for fixed values of $\varepsilon$ and~$\xi$,~$\varepsilon_{\Xdelta}$ may be negative.

    \item[(C2)] If $|S_{\text{interior}}|$ is below the sample complexity as defined in Eq.~\eqref{eq:sample_complexity}.
\end{itemize}
If one of these cases holds (Line~\ref{line:number of sample condition}), the algorithm halts and returns failure (Line~\ref{line:fail}).

If none of these two cases hold, 
all samples in $S_{\text{interior}}$ are labelled as collision-free or not\footnote{In practice, we call \texttt{CD} once for each sample and not multiple times as described in Alg.~\ref{algo:Learning Phase}, Lines~\ref{line:S interior}, and~\ref{line:LCD}. This is done by caching results and not detailed in order to simplify the exposition.}.
These labeled samples are then used to train a HardSVM  classifier (with a linear kernel) after computing the set of features~$\phi$ (Eq.~\eqref{eq:feature_vec}) for each sample (Line~\ref{line:LCD}).
Finally, the learned collision detector \texttt{LCD} is returned (Line~\ref{line:return LCD}).

\begin{lemma}[LCD existence]
    \label{lemma:lcd_loss}
    Given an error bound~$\varepsilon$ and confidence level~$1 - \xi$, there exist a clearance~$\delta$ and a number of samples~$m$ such that Alg.~\ref{algo:Learning Phase} does not return `Fail'.
    For such values of~$\delta$ and~$m$, \texttt{LCD}, the learned collision detector returned, guarantees with confidence at least~$1 - \xi$ that the error on a random sample from $\mathcal{X}$ does not exceed~$\varepsilon$.
\end{lemma}

\begin{proof}
    Recall, that Thm.~\ref{thm:sample_complexity_delta} only provides bounds on how Hard-SVM will classify samples in the $\delta$-interior. 
    Thus, the error $\varepsilon$ of a classifier on a sample~$x \in \mathcal{X}$ will be
    \begin{equation}
    \label{eq:error-decomposition}
    \varepsilon := 
        \text{Pr}[x \in \Xdelta] \cdot \varepsilon_{\Xdelta}
        + 
        \text{Pr}[x \in \Xnotdelta] \cdot \varepsilon_{\Xnotdelta}.    
    \end{equation}
    Here,~$\text{Pr}[x \in \Xdelta]$ and $\text{Pr}[x \in \Xnotdelta]$ are the probability of sampling in $\Xdelta$ and $\Xnotdelta$, respectively and 
    $\varepsilon_{\Xdelta} \in [0,1]$ and~$\varepsilon_{\Xnotdelta}\in [0,1]$ are the error of classifying a sample if it lies in~$\Xdelta$ and $\Xnotdelta$, respectively.

    Denote $p_{\Xdelta} = \text{Pr}[x \in \Xdelta]$, 
    and using the fact that
    $\varepsilon_{\Xnotdelta} \leq 1$
    and that we require
    $\varepsilon_{\Xdelta} \geq 0$,
    we must have that 
    $ p_{\Xdelta} \geq 1 - \varepsilon$.
    Now, set $\delta_{\max}$ to be the clearance\footnote{Strictly speaking, $\delta_{\max}$ is a function of the C-space and should be written as $\delta_{\max}(\mathcal{X})$, but for brevity we use $\delta_{\max}$ and not $\delta_{\max}(\mathcal{X})$.} for which 
    $ p_{\mathcal{X}^{\delta_{\max}}} = 1 - \varepsilon$.
    Note that for a given a C-space $\mathcal{X}$ and an error $\varepsilon$, for any clearance $\delta < \delta_{\max}$ we have that $\varepsilon_{\Xdelta} > 0$.
    
    We will now show that there exists a clearance~$\delta^* <~\delta_{\max}$ and a number of samples $m^*$ (which is a function of~$\delta^*$) for which Lemma~\ref{lemma:lcd_loss} holds which will complete the proof.
    Specifically, we will show that for values $\delta^*$ and~$m^*$ 
    (i)~Conditions (C1) and (C2) do not hold
    and that
    (ii)~\texttt{LCD} guarantees with confidence at least~$1 - \xi$ that the error on a random sample from $\mathcal{X}$ does not exceed~$\varepsilon$. This is an immediate application of Eq.~\ref{eq:error-decomposition} and Thm.~\ref{thm:sample_complexity_delta}.
    
    To show that Condition (C1) does not hold, 
    we start with a simple-yet-subtle observation: 
    for the clearances $\delta = 0$ and $\delta = \delta_{\max}$, the corresponding probability of sampling in the $\delta$-interior is~$p_{\mathcal{X}^{0}} = 1$ and  
    $p_{\mathcal{X}^{\delta_{\max}}} = 1 - \varepsilon < 1$, respectively. Thus, if we decrease $\delta$ from $\delta_{\max}$ to zero, the corresponding probability of sampling in the $\delta$-interior is a weakly monotonically increasing function.
    Thus, we set~$\delta^*$ to be some clearance such that 
    (i)~$0 <~\delta^* <~\delta_{\max}$
    and
    (ii)~$p_{\mathcal{X}^{\delta^*}} - p_{\mathcal{X}^{\delta_{\max}}} = \Delta^*$
    for some $\Delta^{*} > 0$ and note that such $\delta^*$ always exists.
    Finally, we set $m^*$ to be 
    $m^*:= \max\left( \kappa_1, \kappa_2 \right)$
    where 
    $\kappa_1:= (2 z_{\xi/2} / \Delta^{*})^2$
    and~{$\kappa_2:=~4m_{\mathcal{X}^{\delta^*}}(\varepsilon_{\mathcal{X}^{\delta^*}}, \xi ) / \left(4(1 - \varepsilon) + {3\Delta^*}\right)$}. 
    
    To show that Condition (C1) does not hold, we need to show that the numerator~$\varepsilon_{\mathcal{X}^{\delta^*}}^{\text{numer}}$ of Eq.~\eqref{eq:classification_error} is positive.
    \begin{align*}
        \varepsilon_{\mathcal{X}^{\delta^*}}^{\text{numer}} 
        &:=
            \varepsilon - (1 - \hat{p}_{\mathcal{X}^{\delta^*}}) - z_{\xi/2} \cdot \sqrt{\frac{\hat{p}_{\mathcal{X}^{\delta^*}} (1 - \hat{p}_{\mathcal{X}^{\delta^*}})}{m^*}} \\
        &\overbrace{\geq}^{(1)} 
            p_{\mathcal{X}^{\delta^*}} - (1 - \varepsilon) - 2\cdot z_{\xi/2} \cdot \sqrt{\frac{\hat{p}_{\mathcal{X}^{\delta^*}}(1 - \hat{p}_{\mathcal{X}^{\delta^*}})}{m^*}} \\
        &\overbrace{\geq}^{(2)} 
            p_{\mathcal{X}^{\delta^*}} - (1 - \varepsilon) - 2\cdot z_{\xi/2} \cdot \sqrt{\frac{1}{4m^*}} \\
        &\overbrace{=}^{(3)} 
            p_{\mathcal{X}^{\delta_{\max}}} - (1 - \varepsilon) + \Delta^* - 2\cdot z_{\xi/2} \cdot \sqrt{\frac{1}{4m^*}} \\
        &\overbrace{=}^{(4)} 
            \Delta^* - 2\cdot z_{\xi/2} \cdot \sqrt{\frac{1}{4m^*}}
        \overbrace{\geq}^{(5)} 
            \Delta^* / 2  > 0.
    \end{align*}
    Here, 
    {(1)~is an application of \conference{the binomial proportion confidence interval~\cite{agresti1998approximate, brown2002confidence} which guarantees that~$p_{\mathcal{X}^{\delta^*}} \leq~\hat{p}_{\mathcal{X}^{\delta^*}} +~z_{\xi/2} \cdot~\sqrt{(\hat{p}_{\mathcal{X}^{\delta^*}}\cdot (1 - \hat{p}_{\mathcal{X}^{\delta^*}}))/m}$}{Lemma~\ref{lemma:binomial_bounds}},
    (2)~uses the fact that~$\forall p \in~[0,1]~p \cdot(1-p) \leq 0.25$,
    (3)~follows from the definition of~$\delta_{\max}$ which ensures that~$p_{\mathcal{X}^{\delta_{\max}}} =~1 -~\varepsilon$,
    (4)~follows from the definition of~$\Delta^*$,
    (5)~follows from the fact that $m^* \geq \kappa_1$.
    Finally, we use the fact that $\Delta^* > 0 $.}
    
    Now, to show that Condition (C2) does not hold, will show that {$
    |S_{\text{interior}}| \geq m_{\mathcal{X}^{\delta^*}} (\varepsilon_{\mathcal{X}^{\delta^*}}, \xi)
    $}. 
    %
    Specifically, we have:    
    \begin{align*}
    |S_{\text{interior}}| &:= m^* \cdot \hat{p}_{\mathcal{X}^{\delta^*}} \\
    &\overbrace{\geq}^{(6)} m^* \cdot (p_{\mathcal{X}^{\delta^*}} - z_{\xi/2} \cdot \sqrt{\frac{\hat{p}_{\mathcal{X}^{\delta^*}}(1 - \hat{p}_{\mathcal{X}^{\delta^*}})}{m^*}}) \\
    &\overbrace{\geq}^{(7)} m^* \cdot (p_{\mathcal{X}^{\delta^*}} - z_{\xi/2} \cdot \sqrt{\frac{1}{4m^*}}) \\
    &\overbrace{\geq}^{(8)} m^* \cdot (p_{\mathcal{X}^{\delta^*}} - \frac{\Delta^*}{4}) \\
    &\overbrace{=}^{(9)}  m^* \cdot ((1 - \varepsilon) + \frac{3\Delta^*}{4})
    \overbrace{\geq}^{(10)}  {m_{\mathcal{X}^{\delta^*}} (\varepsilon_{\mathcal{X}^{\delta^*}}, \xi).}
    \end{align*}
    Here, 
    {(6)~is another an application of \conference{the binomial proportion confidence interval}{Lemma~\ref{lemma:binomial_bounds}},}
    (7)~again uses the fact that~$\forall p \in~[0,1]~p\cdot~(1-p) \leq 0.25$,
    (8)~follows from the fact that $m^* \geq \kappa_1$,
    (9)~follows from the definition of $\Delta^*$ and the fact that~$p_{\mathcal{X}^{\delta_{\max}}} = 1 - \varepsilon$
    and
    (10)~follows from the fact that~$m^* \geq \kappa_2$,

    Since both Conditions (C1) and (C2) do not hold for our choice of $\delta^*$ and $m^*$, Alg.~\ref{algo:Learning Phase} does not return 'Fail' which completes the proof of Lemma~\ref{lemma:lcd_loss}.
\end{proof}

\conference{}{
\vspace{2mm}
\noindent
\textbf{Note.} To simplify exposition, Lemma~\ref{lemma:lcd_loss} does not include the distribution used to sample configurations during CD (classification). Here, we assume that this is the same distribution used in the learning phase.
}

\begin{figure}[t]
    \centering
    \subfloat[]{\includegraphics[width=0.32\linewidth]{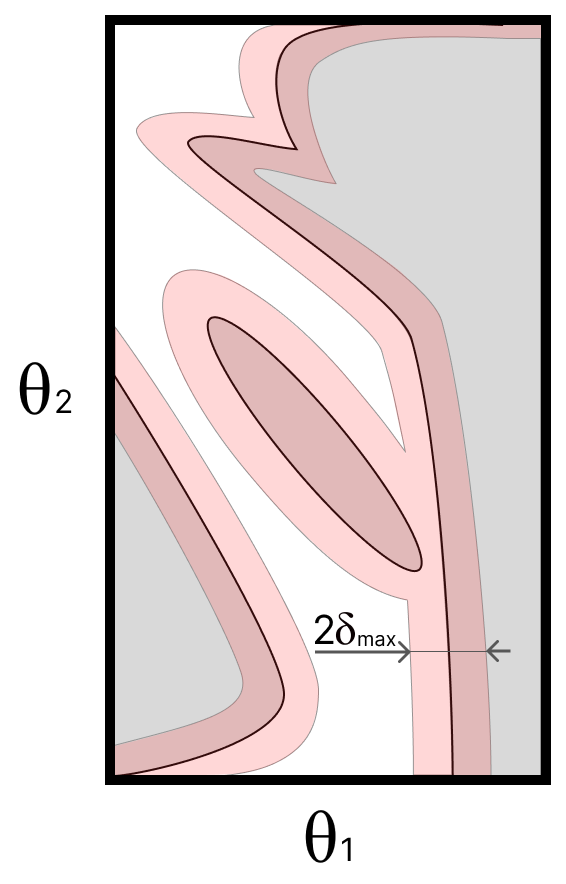}%
        \label{fig:delta_interior_max}}
    \hfill
    \subfloat[]{\includegraphics[width=0.32\linewidth]{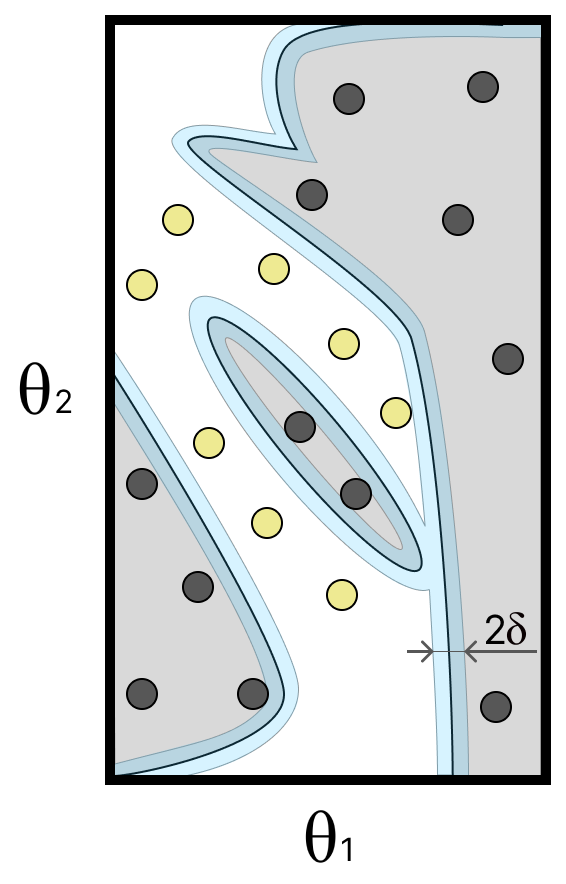}%
        \label{fig:delta_star_samples}}
    \hfill
    \subfloat[]{\includegraphics[width=0.32\linewidth]{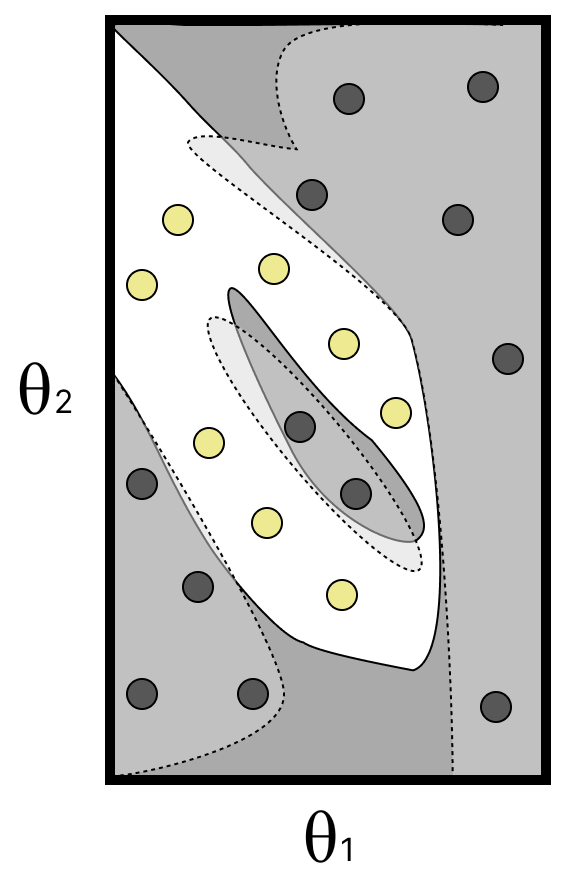}%
        \label{fig:lcd_learned_space}}
    \caption{{Illustration of the learning process for the C-space visualized in Fig.~\ref{fig:workspace_delta} depicting~$\Xfree$ 
    (\fcolorbox{black}{white}{\rule{0pt}{3pt}\rule{3pt}{0pt}})
    and~$\Xforb$
    (\fcolorbox{black}{my_gray}{\rule{0pt}{3pt}\rule{3pt}{0pt}}): 
    (a)~The $\delta$-boundary for $\delta = \delta_{\max}$ (\fcolorbox{black}{my_red}{\rule{0pt}{3pt}\rule{3pt}{0pt}}), where Alg~\ref{algo:Learning Phase} fails because~$\varepsilon_{\Xdelta} \leq 0$ (C1 holds). 
    (b)~The $\delta$-boundary for~$\delta <~\delta^*$~(\fcolorbox{black}{my_cyan}{\rule{0pt}{3pt}\rule{3pt}{0pt}})
    and $m > m^*$ samples, that all lie in the $\delta$-interior for learning the CD, where (C1) and (C2) do not hold. 
    (c) The space learned by the LCD with the boundary between $\Xforb$ and $\Xforb$ dashed and highlighting $\hat{\mathcal{X}_{\text{forb}}}$ 
    (\fcolorbox{black}{my_dark_gray}{\rule{0pt}{3pt}\rule{3pt}{0pt}}), the part estimating $\Xforb$.
    }}
    \label{fig:learning_process}
    \vspace{-5mm}
\end{figure}

For an illustration of this learning process, see Fig.~\ref{fig:learning_process}.

\begin{algorithm}[t]
\caption{Learning-based Collision Detection}
\label{algo:AdaptiveLearning}
\begin{algorithmic}[1]
\Require 
    $\varepsilon$, 
    $\xi$, 
    $\mathcal{X} = [0,1]^d$, 
    \texttt{CD} 
\State $m \gets m_0$, \hspace{3mm} $\delta \gets \delta_0$ \Comment{initial guesses} \label{line:guesses m and delta}

\While{$(\delta,m)\text{-LBCD}(\varepsilon, \xi, \delta, m, \mathcal{X}, \texttt{CD}) = \text{`Fail'}$} \label{line:if lbcd}
    \State $m \gets 2 \cdot  m$, \hspace{3mm} $\delta \gets \delta / 2$ \label{line:update m and delta}

\EndWhile

\State \texttt{LCD} $\gets (\delta,m)\text{-LBCD}(\varepsilon, \xi, \delta, m, \mathcal{X}, \texttt{CD})$ \label{line:get lcd}
\State \Return \texttt{LCD}
\end{algorithmic}
\end{algorithm}

\subsection{LCD with Guarantees}
\label{subsec:lcd2}
Our LCD algorithm (Alg.~\ref{algo:AdaptiveLearning}), receives as input 
a desired error bound~$\varepsilon$, a confidence level~$1 - \xi$, 
and an exact collision detector \texttt{CD}.
It begins with initial guesses for the clearance threshold~$\delta_0$ and sample size~$m_0$ (Line~\ref{line:guesses m and delta}).
At each iteration, the algorithm attempts to learn a CD using these values.

The algorithm calls the $(\delta, m)$-\texttt{LBCD} function (Alg.~\ref{algo:Learning Phase}) with the current values of~$\delta$ and~$m$ (Line~\ref{line:if lbcd}). If the function fails (returns `Fail'), it indicates that either the sample size is too small or the clearance threshold is too large to achieve the desired error bounds. 
In this case, the algorithm doubles the sample size and halves the clearance threshold (Line~\ref{line:update m and delta}). 

This iterative process continues until the function~$(\delta, m)$-\texttt{LBCD} succeeds in computing a learned collision detector  that satisfies the error bound~$\varepsilon$ and the confidence level~$1 - \xi$.
Once the conditions are met, the final collision detector~\texttt{LCD} is returned (Line~7).

\begin{lemma}[Convergence of Adaptive Learning Algorithm]
\label{lemma:convergence}
For any given error bound $\varepsilon > 0$ and confidence level $1 - \xi$, Alg.~\ref{algo:AdaptiveLearning} terminates in a finite number of iterations and returns a learned collision detector (LCD) that guarantees, with confidence at least $1 - \xi$, that the error on a random sample from $\mathcal{X}$ does not exceed $\varepsilon$.
\end{lemma}

\begin{proof}
Let $\delta'$ and $m'$ be clearance values and number of samples as guaranteed by Lemma~\ref{lemma:lcd_loss}, and note that the same guarantees hold for any $\delta \leq \delta'$ and $m \geq m'$.
After~$\max 
    \left(
        \lceil \log_2(\delta_0 / \delta') \rceil,
        \lceil \log_2(m' / m_0) \rceil
    \right)
$
iterations, Alg.~\ref{algo:AdaptiveLearning} will call Alg.~\ref{algo:Learning Phase} with clearance values and number of samples such that 
 with confidence at least $1 - \xi$, the error on a random sample from $\mathcal{X}$ of the LCD returned will not exceed $\varepsilon$.
 \end{proof}

\section{Empirical Results}
\label{sec:eval}

In this section, we illustrate the relationship between the different parameters of our theoretical framework and analysis. In our proofs we often need to balance the effects of different parameters in order to guarantee some behavior (e.g., $m_{\Xdelta}$ and $\varepsilon_{\Xdelta}$ in Alg.~\ref{algo:Learning Phase}) and here we visualize this trade off. All results are for the scenario depicted in Fig.~\ref{fig:workspace_delta}.

Fig.~\ref{fig:p_vs_delta} illustrates the relationship between the size $|\mathcal{X}^\delta|$ of the $\delta$-interior and the clearance value $\delta$. 
Recall that the proof sketch of Lemma~\ref{lemma:lcd_loss}  emphasizes the need to decrease $\delta$ below some critical threshold in order for $|\mathcal{X}^\delta|$ to be large enough.
Indeed, we can see empirically that as $\delta$ decreases,  $|\mathcal{X}^\delta|$ increases. Intuitively, this is because  smaller clearance values result in a larger portion of the C-space being considered as part of the $\delta$-interior. Conversely, as $\delta$ increases, the proportion of the $\delta$-interior decreases, as fewer configurations meet the stricter clearance requirement.
Importantly, this graph demonstrates the existence of a critical clearance value~$\delta^*$  between zero and~{$\delta_{\max}$. 
\addExample{For example, an error bound of $\varepsilon = 0.1$ corresponds with $|\mathcal{X}^\delta| {>} 0.9$ (i.e., $1 - \varepsilon$) which implies that $\delta^* {<} 0.028$.}

Fig.~\ref{fig:epsilon_vs_delta} demonstrates how the classification error~$\varepsilon_{\mathcal{X}^{\delta}}$ (Alg.~\ref{algo:Learning Phase}, Line~\ref{line:epsilon interior}) varies with $\delta$ for different overall error bounds~$\varepsilon$. 
We observe that for all $\varepsilon$ values, $\varepsilon_{\mathcal{X}^{\delta}}$ weakly monotonically decreases as $\delta$ increases
and given a fixed clearance $\delta$, larger~$\varepsilon$ values correspond to larger $\varepsilon_{\Xdelta}$

Importantly, we see that $\varepsilon_{\mathcal{X}^{\delta}}$ approaches zero as $\delta$ nears its maximum value for each $\varepsilon$.
This corresponds to the setting where $|\Xdelta|$  approaches $1 - \varepsilon$, beyond which our algorithm cannot provide guarantees. 
\addExample{
For example, when $\varepsilon = 0.1$, we see $\varepsilon_{\mathcal{X}^{\delta}}$ approaching zero near $\delta \approx 0.028$, consistent with our observation from Fig.~\ref{fig:p_vs_delta}.}

Fig.~\ref{fig:m_vs_delta} depicts the relationship between the sample complexity $m_{\mathcal{X}^{\delta}}$ and $\delta$ for different $\varepsilon$ values. As discussed in Sec.~\ref{subsec:discussion1}, $m_{\mathcal{X}^{\delta}}$ approaches infinity as $\delta$ approaches either zero or $\delta^*$, motivating our approach for choosing $\delta$ in Alg.~\ref{algo:AdaptiveLearning}.

\addExample{
Going back to our running numeric example of $\varepsilon = 0.1$ with $\delta^* \leq 0.028$ from Fig.~\subref{fig:p_vs_delta}-\subref{fig:m_vs_delta}, we can see that indeed $m_{\mathcal{X}^{\delta}} \rightarrow \infty$  at zero and around $0.028$ dramatically dropping in between these two extreme values.}

\begin{figure}[t]
\vspace{-4mm}
    \centering
    \subfloat[]{\includegraphics[width=0.29\linewidth]{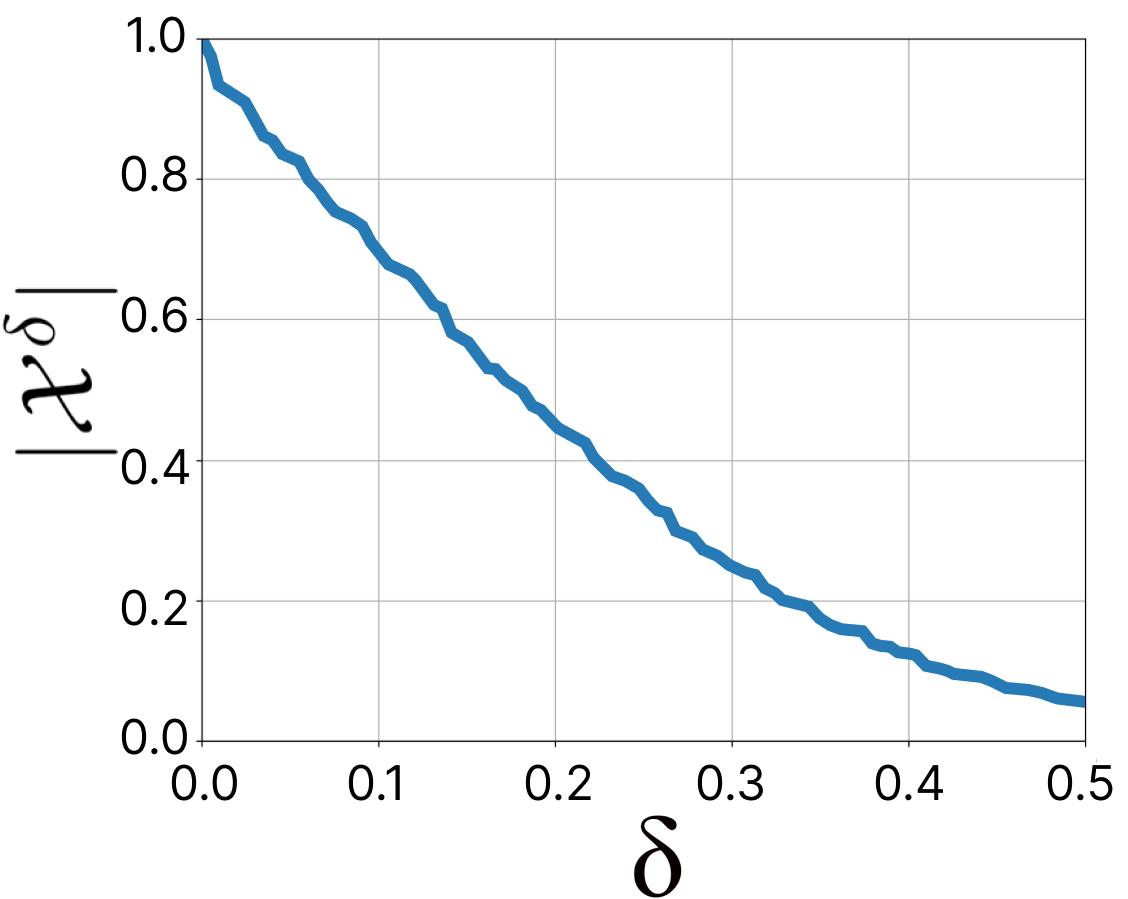}%
        \label{fig:p_vs_delta}}
    \hfill
    \subfloat[]{\includegraphics[width=0.32\linewidth]{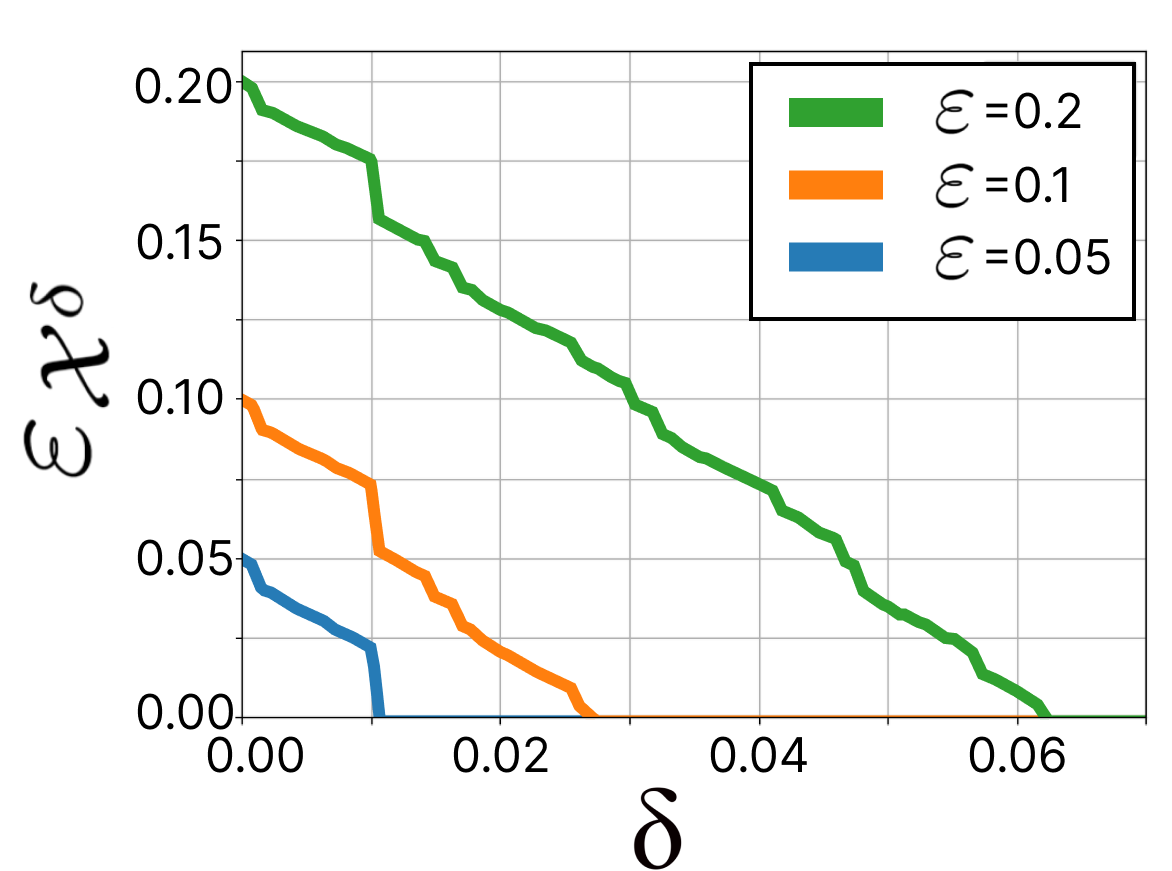}%
        \label{fig:epsilon_vs_delta}}
    \hfill
    \subfloat[]{\includegraphics[width=0.36\linewidth]{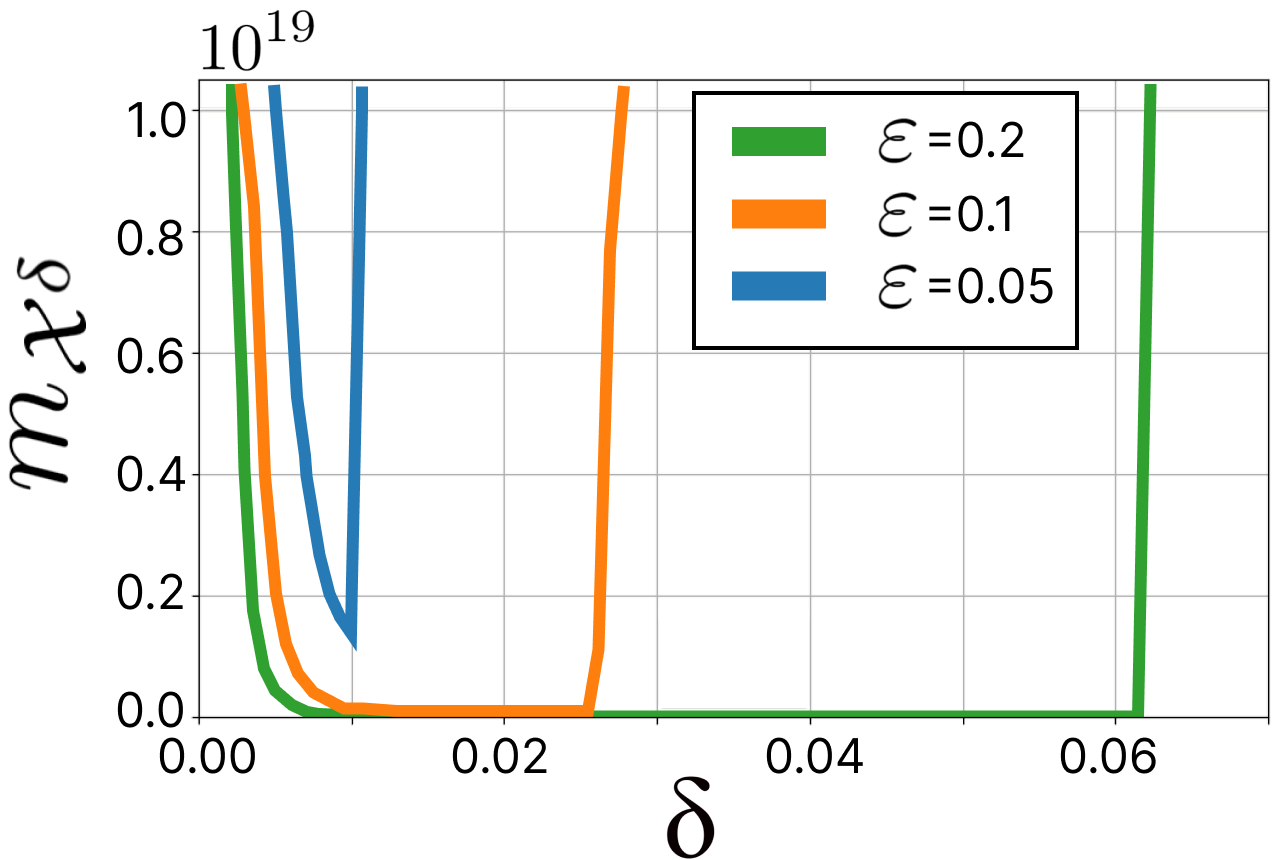}%
        \label{fig:m_vs_delta}}
    \caption{Trends between key parameters in our framework.
    (a)~The size~$\vert \Xdelta \vert$ of the $\delta$-interior as a function of $\delta$. 
    (b)~Classification error~{$\varepsilon_{\Xdelta}$} in $\delta$-interior as a function of $\delta$ for different $\varepsilon$ values. 
    (c)~Sample complexity as a function of $\delta$ for different $\varepsilon$ values.}
    \label{fig:combined_results}
    \vspace{-5mm}
\end{figure}
\section{Conclusion and Future Work}
\label{sec:future}

In this work we presented a theoretical framework that established how to apply theoretical results from the ML literature to SVM-based LCD for MP. This was done by showing the connection between clearance, defined the C-space (where planning occurs) and margin, defined in the feature space (where learning occurs).

This paper focused on establishing formal well-grounded theoretical results. This theoretical nature ignores key aspects making the approach impractical to be used as-is in practical LCDs. 
The crux of the problem is that, similar to some ML-results, our sample-complexity bounds (Eq.\eqref{eq:sample_complexity}) have an exponential dependence 
on the dimensionality $d$ of the C-space and 
on the inverse of the clearance $\delta$. 

Importantly, and despite these limitations, our work introduces LCDs that have formal guarantees on their output. We foresee how such statistical CD can be used within a MP algorithm that can guarantee to produce paths that are collision-free with high probability for any user-defined probability and confidence level. This can be achieved by considering any sequence of configurations and aggregating their individual probabilities of being correctly classified as collision-free. Details are left for future work.

Another avenue for future work includes addressing the practical limitations of our approach. One possible direction is to explore methods to reduce the exponential dependence of the sample-complexity bounds on $d$ and $1/\delta$. 
This could involve selecting different feature mappings with lower dimensionality or investigating alternative ML algorithms that offer similar statistical guarantees but with more favorable computational properties.

\section*{Acknoledgements}
S. Tubbul and O. Salzman were supported in part by the Technion Autonomous Systems Program (TASP) and the Israeli Ministry of Science \& Technology grant No. 3-17385.
Aviv Tamar was Funded by the European Union (ERC, Bayes-RL, 101041250).  Views and opinions expressed are however those of the author(s) only and do not necessarily reflect those of the European Union or the European Research Council Executive Agency (ERCEA). Neither the European Union nor the granting authority can be held responsible for them.

\bibliographystyle{IEEEtran}
\bibliography{IEEEabrv,references}

\conference{}{
\section*{APPENDIX}

\subsection{Statistics}
We present a fundamental result in statistics that will be important in our analysis which is based on a normal approximation of the binomial distribution via the Central Limit Theorem~\cite{agresti1998approximate, brown2002confidence}. 

\begin{lemma}[Binomial Proportion Confidence Interval]
\label{lemma:binomial_bounds}  
Let $\mathcal{X}$ be some space and $\mathcal{X}' \subset \mathcal{X}$ be a subset of~$\mathcal{X}$. 
Let $\mathcal{D}$ be a probability distribution over $\mathcal{X}$, and set~$p_{\mathcal{X}'} = \mathbb{P}_{x \sim \mathcal{D}}[x \in \mathcal{X}']$ to denote the probability that a sample from $\mathcal{X}$, drawn according to $\mathcal{D}$, falls in $\mathcal{X}'$. Given~$m$ independent and identically distributed (i.i.d.) samples\footnote{Importantly, Lemma~\ref{lemma:binomial_bounds} holds under the condition that the sample size $m$ is large enough to justify the normal approximation. This typically occurs for $m \cdot \min(p_{\mathcal{X}'}, 1 - p_{\mathcal{X}'}) \geq 5$~\cite{brown2001interval}.} from~$\mathcal{X}$ according to $\mathcal{D}$, let $m'$ be the number of samples that fall in $\mathcal{X}'$. 

Then, for a confidence level $1 - \xi$, we can bound $p_{\mathcal{X}'}$ by:
\[
p_{\mathcal{X}'} \leq \hat{p}_{\mathcal{X}'} + z_{\xi/2} \cdot \sqrt{\frac{\hat{p}_{\mathcal{X}'} (1 - \hat{p}_{\mathcal{X}'})}{m}},
\]
where $\hat{p}_{\mathcal{X}'} := m' / m$ and $z_{\xi/2}$ is the critical value of the standard normal distribution corresponding to the confidence level $1 - \xi$.
\end{lemma}

\subsection{Proof of Thm.~\ref{thm:Feature-Space Margin}}

We begin by defining the mapping~$g:~[n]^d\rightarrow~\{-1,0,+1\}$ that returns for each cell a label $\pm 1$ if all configurations are in the $\delta$-interior of~$\mathcal{X}$ and zero otherwise:

\begin{equation}
\label{eq:g}
g(\mathbf{i}) = \begin{cases}
        +1, & \text{if } \exists \mathbf{x} \in \textsc{Cell}_{\mathbf{i}}: f(\mathbf{x})=+1, \\
        -1, & \text{if } \exists \mathbf{x} \in \textsc{Cell}_{\mathbf{i}}: f(\mathbf{x})=-1, \\
        0, & \text{otherwise.}
         \end{cases}
\end{equation}

Now, we set
\begin{align*}
    \mathbf{\alpha} := ( &g(1,1,\ldots,1),  g(2,1,\ldots,1), \ldots, g(n,1,\ldots,1),\\
                         &g(1,2,\ldots,1),  g(2,2,\ldots,1), \ldots, g(n,2,\ldots,1),\\
                         &\ldots\\
                         &g(n,n,\ldots,1),  g(n,n,\ldots,2), \ldots, g(n,n,\ldots,n)).
\end{align*}

Note that following our choice of $n$, $g$ is well defined. {Namely, that for each cell, there no two  samples $x_1,x_2$ exist such that $f(x_1)=1$ and $f(x_2)=-1$. This ensures that each cell can be unambiguously labeled as either entirely in the free space, entirely in the forbidden space, or in the boundary region.}

Following Def.~\ref{def:margin}, we will show that there exists some~$\gamma >~0$ such that $\forall x \in \mathcal{X}$, if $f(\mathbf{x})=1$ then $\mathbf{\alpha} \cdot \phi(\mathbf{x}) \geq \gamma$, and if $f(\mathbf{x})=-1$ then $\mathbf{\alpha} \cdot \phi(\mathbf{x}) \leq -\gamma$.

\textbf{Case 1 ($f(\mathbf{x})=1$):}
Assume, without loss of generality, that $\mathbf{x}$ lies within the multidimensional interval defined by~$\prod_{k=1}^{d}\left[\frac{j_k-1}{n}, \frac{j_k}{n} \right)$ and set $\mathbf{j}:= \langle j_1,j_2,\ldots,j_d \rangle$. We have:

\begin{align*}
\mathbf{\alpha}\phi(\mathbf{x}) 
    &= \sum_{\mathbf{i}\in [n]^d}{g(\mathbf{i}) \cdot h_{\mathbf{i}}(\mathbf{x})}\\
    &=  g(\mathbf{j}) \cdot  h_{\mathbf{j}}(x) +
        \sum_{\substack{\mathbf{j}\neq \mathbf{i} \in [n]^d}} g(\mathbf{i}) \cdot h_{\mathbf{i}}(\mathbf{x})\\
    &\stackrel{(1)}{=} 1 \cdot \exp\left(-\frac{\|\textsc{Center}_{\mathbf{j}} - \mathbf{x}\|^2}{\sigma^2}\right) \\
    &\quad + \sum_{\substack{\mathbf{j}\neq \mathbf{i} \in [n]^d}} 
            g(\mathbf{i}) 
                \cdot 
            \exp\left(-\frac{\|\textsc{Center}_{\mathbf{i}} - \mathbf{x}\|^2}{\sigma^2}\right)\\
    &\stackrel{(2)}{\geq} \exp\left(-\frac{\|\textsc{Center}_{\mathbf{j}} - \mathbf{x}\|^2}{\sigma^2}\right) \\
    &\quad - n^d\cdot \exp\left(-\frac{(2\delta - \frac{\sqrt{d}}{2n})^2}{\sigma^2}\right) \\
    &\stackrel{(3)}{\geq} \exp\left(-\frac{\left(\frac{\sqrt{d}}{2n}\right)^2}{\sigma^2}\right) - 
        n^d\cdot \exp\left(-\frac{(2\delta - \frac{\sqrt{d}}{2n})^2}{\sigma^2}\right) \\
    &\stackrel{(4)}{=} \exp\left(-\frac{\left(0.5\delta\right)^2}{\frac{2\delta^2}{\ln(9n^d)}}\right) - 
        n^d\cdot \exp\left(-\frac{\left(1.5\delta\right)^2}{\frac{2\delta^2}{\ln(9n^d)}}\right) \\
    &\stackrel{(5)}{=} \frac{8}{9^{\frac{9}{8}} n^{\frac{d}{8}}} > 0.
\end{align*}

(1) Applies $g(\mathbf{j})=1$ for $\mathbf{x} \in \textsc{Cell}_{\mathbf{j}}$ with $f(\mathbf{x})=1$, and expands $h_{\mathbf{i}}(\mathbf{x})$.

(2) Lower-bounds the sum by setting $g(\mathbf{i}) = -1$ for all cells except $\mathbf{j}$. $(2\delta - \frac{\sqrt{d}}{2n})$ is the minimum distance to a negative cell's center.

(3) Uses $\frac{\sqrt{d}}{2n}$ as the maximum distance between any point in a cell and its center.

(4) Substitutes $n = \frac{\sqrt{d}}{\delta}$ and $\sigma^2 = \frac{2\delta^2}{\ln(9n^d)}$ from the theorem.

(5) Evaluates the expressions, proving the existence of a positive margin.

Since $\mathbf{\alpha}$ is a vector where every value can only take one of the options $\{0,1,-1\}$, in the worst case $|\mathbf{\alpha}|=\sqrt{n^d}$. Hence, by normalizing the term, we get:

\begin{align*}
    \mathbf{\alpha}\phi(\mathbf{x}) &\geq \frac{8}{9^{\frac{9}{8}} n^{\frac{5d}{8}}}, \text{ where } |\mathbf{\alpha}|=1.
\end{align*}

\textbf{Case 2 ($f(\mathbf{x})=-1$):}
This case is symmetrical to Case 1 and results in $\mathbf{\alpha}\phi(\mathbf{x}) \leq -\frac{8}{9^{\frac{9}{8}} n^{\frac{5d}{8}}} < 0$, where $|\mathbf{\alpha}|=1$.

Therefore, we have shown that $\Phi$ contains a positive margin $\gamma^* \geq \frac{8}{9^{\frac{9}{8}} n^{\frac{5d}{8}}}$.

\subsection{Proof of Thm.~\ref{thm:sample_complexity_delta}}
We start with the error bound from Thm. 15.4 in \cite{shalev2014understanding}:

\begin{equation}
\label{eq:error_bound}
\varepsilon \leq \sqrt{\frac{4(\frac{\rho'}{\gamma'})^2}{m_{\Xdelta}}} +  \sqrt{\frac{2\ln{\frac{2}{\xi}}}{m_{\Xdelta}}}.
\end{equation}

where $\varepsilon$ is the error, $m$ is the sample size, $\rho'$ is the radius of the ball containing the support of the distribution, $\gamma'$ is the margin, and $\xi$ is the confidence parameter.

Let $a = \frac{4(\frac{\rho'}{\gamma'})^2}{m_{\Xdelta}}$ and $b = \frac{2\ln{\frac{2}{\xi}}}{m_{\Xdelta}}$. Applying the inequality~$\sqrt{a} + \sqrt{b} \leq 2\sqrt{a + b}$, we get:

\begin{equation}
\varepsilon \leq 2\sqrt{\frac{4(\frac{\rho'}{\gamma'})^2 + 2\ln{\frac{2}{\xi}}}{m_{\Xdelta}}}.
\end{equation}

Algebraic manipulations to isolate $m$:

\begin{equation}
m_{\Xdelta} \leq \frac{8}{\varepsilon^2} \left[2\left(\frac{\rho'}{\gamma'}\right)^2+\ln\left(\frac{2}{\xi}\right)\right].
\end{equation}

Now, we substitute the values of $\rho'$, $\gamma'$ and $n$ from Thm.~\ref{thm:Feature-Space Margin}:
$$
\rho' = \sqrt{n^d} = n^{\frac{d}{2}} \text{(since} \|\phi(\mathbf{x})\|^2 \leq n^d\text{)}
$$
$$
\gamma' \geq \frac{8}{9^{\frac{9}{8}} n^{\frac{5d}{8}}}
$$
$$
n = \frac{\sqrt{d}}{\delta}
$$

Substituting these into our bound:

\begin{equation}
m_{\Xdelta} \leq \frac{1}{\varepsilon^2} \left[
    \frac{9^{\frac{9}{4}}}{4}\left(\frac{\sqrt{d}}{\delta}\right)^{\frac{9d}{4}}+8\ln\left(\frac{2}{\xi}\right)
\right].
\end{equation}

This completes the proof, providing an upper bound on the sample complexity $m$.
}


\end{document}